\documentclass{article}

\usepackage[dblblindworkshop, final]{neurips_2026}

\usepackage[utf8]{inputenc}
\usepackage[T1]{fontenc}
\usepackage{url}
\usepackage{booktabs}
\usepackage{amsfonts}
\usepackage{nicefrac}
\usepackage{microtype}
\usepackage{xcolor}
\usepackage{graphicx}

\usepackage{amsmath}
\usepackage{amssymb}
\usepackage{mathtools}
\usepackage{amsthm}
\usepackage{subcaption}
\usepackage{multirow}
\usepackage{algorithm}
\usepackage{algorithmic}

\usepackage{wrapfig}
\usepackage{listings}
\usepackage{hyperref}

\theoremstyle{plain}
\newtheorem{theorem}{Theorem}[section]
\newtheorem{proposition}[theorem]{Proposition}

\theoremstyle{definition}

\theoremstyle{remark}
\newtheorem{remark}[theorem]{Remark}

\title{Belief-Conditioned Multi-Agent Actor Critic for Decentralized LLM
Collaboration}

\workshoptitle{PTA: From Pretrained Representations to Acting Agents}

\author{%
  Tianle Chen \qquad Shuo Liu \qquad Christopher Amato\thanks{Corresponding author.} \\
  Khoury College of Computer Sciences, Northeastern University, Boston, MA, 02215, USA \\
  \texttt{\{chen.tianle, liu.shuo2, c.amato\}@northeastern.edu} \\
}

\begin{document}

\maketitle

\begin{abstract}
  Recent work has explored optimizing LLM collaboration through Multi-Agent Reinforcement Learning (MARL), but many approaches still rely on predefined workflows, limiting efficiency and scalability. Decentralized LLM collaboration is challenging because agents execute independently with limited communication, share a joint objective, and must infer useful information from communication traces. In this paper, we propose \textbf{BCMAAC}, a \textbf{B}elief-\textbf{C}onditioned \textbf{M}ulti-\textbf{A}gent \textbf{A}ctor-\textbf{C}ritic method that maintains beliefs over contextual information and uses them to facilitate training. Experiments on Minecraft game-playing, web searching and tool-using tasks show that BCMAAC improves task performance and collaborative behavior compared to naive actor-critic baselines. The improvements are particularly pronounced in multi-turn tasks with richer contextual information, where belief modeling provides a clear advantage.
\end{abstract}

\section{Introduction}

Recent advances in LLMs have motivated researchers to use them as intelligent agents and interact with environment in multiple turns \cite{yao2022react}. Many works use reinforcement learning with verifiable rewards (RLVR), where rewards are derived from verifiable signals such as rules and unit tests~\citep{lambert2024tulu3,guo2025deepseekr1,kimiteam2025kimi,liu2025magrpo,guo2025deepseekr1}. Compare with reinforcement learning from human feedback (RLHF) that trains a reward model using human preference data, RLVR typically has lower cost, greater scalability, and more objective feedback in scenarios where task success can be automatically verified~\citep{lambert2024tulu3,guo2025deepseekr1}.

Existing works have employed multi-LLM collaboration to improve performance. Prompt-based approaches for multi-agent systems have demonstrated improved problem-solving through role assignment and agent-to-agent communications~\citep{du2024debate,qian2024chatdev,wu2023autogen}. However, they often follow a manually designed interaction protocols. Recent work has combined reinforcement learning approaches with LLM collaboration. For instance, CoLLM-CC~\citep{liu2026collm} is a multi-agent actor-critic (MAAC) methods that utilize centralized value estimation while keeping decentralized execution.

Despite the advantages of MAAC in LLM collaboration, it faces many challenges. Given the partial observability of the task, the policy is a pre-trained LLM, and actions are generated as sequences of natural language tokens. Action selection relies not only on task reasoning but also on inferences about the environment and the other agents' internal states, thereby increasing the difficulty of decision-making and optimization. Furthermore, the problem of long-horizon sparse rewards is another challenge, making it difficult to capture effective learning signals in a timely manner. This manifests as rewards that are typically delayed, sparse, and concentrated in the terminal state, making it hard for the model to determine which intermediate communications, inferences, or the environmental interactions indeed contributed to the final success.

We study decentralized LLM collaboration optimized with RLVR, focusing on environmental state estimation and the other agents' intention inference. We propose BCMAAC, a belief-conditioned multi-agent actor-critic method. BCMAAC uses a recurrent belief encoder to maintain latent belief context from structured side information and conditions this context onto each LLM actor; during the training phase, a centralized critic aggregates the multi-agent belief context to estimate the value, while the execution phase remains decentralized, with each agent relying solely on its own history and local belief context for decision-making.

Our contributions can be summarized as follows. 
(i) We propose BCMAAC. BCMAAC learns execution-time belief context from structured side information and uses it for both actor conditioning and centralized value estimation. 
(ii) We show that explicit belief context is particularly useful for long-horizon, partially observable collaboration, where useful actions depend on persistent hidden state, the specific information of other agents, and delayed or sparse feedback. This gives agents an intermediate representation for deciding what to do and communicate across turns. 
(iii) Experiments on collaboration tasks show that BCMAAC outperforms MAAC. And the ablation show that both task-belief and partner-action auxiliary signals help improve performance.

\section{Related Works}

\paragraph{LLM-based multi-agent collaboration.}
Recent work uses multiple LLM agents to solve tasks that are difficult for a single model. Multi-agent debate improves the answer by independent reasoning and iterative refinement~\citep{du2024debate}. ChatDev organizes role-specialized agents along a structured software-development pipeline~\citep{qian2024chatdev}. AutoGen provides a general framework in which conversable agents cooperate with tools and optional human input~\citep{wu2023autogen}. CAMEL relies on role-playing between agents, which cooperate through prompting~\citep{li2023camel}. These systems demonstrate the benefits of communication and task decomposition. However, their collaboration strategies are largely specified by prompts, leaving the collaboration pattern fixed and unable to improve with task experience.

\paragraph{RL for LLM collaboration.}

Besides prompt engineering, recent studies also train multi-LLM systems. MAGRPO jointly trains cooperating LLM agents on joint rewards without a critic~\citep{liu2025magrpo}. Dr.~MAS investigates stability issues arising in GRPO-style critic-free training of multi-agent LLM systems. They improve performance by normalizing each agent's advantage with its own reward statistics~\citep{feng2026drmas}. MAPoRL trains multiple LLMs with verifier-based rewards to encourage collaborative behavior~\citep{park2025maporl}. MARFT gives a general formulation for reinforcement fine-tuning of LLM multi-agent systems~\citep{liao2025marft}. Most closely related is MAAC, which formulates LLM collaboration as cooperative MARL. It adopts centralized training with decentralized execution, with a centralized critic and shared rewards, for long-horizon collaboration~\citep{kraemer2016ctde,liu2026collm}. Unlike prompt-defined systems, these methods allow collaboration policies to adapt to task feedback. Nevertheless, they usually use raw interaction histories as context. The actor and critic must reconstruct task state and the knowledge of other agents from long textual histories, making representation learning difficult.

\paragraph{Partner modeling and belief learning in MARL.}

Early methods condition an agent's policy on observed opponent behavior or infer another agent's hidden goal~\citep{he2016deepopp,raileanu2018som}. Later work learns more general representations of agents' latent characteristics from their behavioral histories~\citep{rabinowitz2018tomnet,papoudakis2021liam}. Also some research improves agent collaboration by learned communication~\citep{foerster2016dial,sukhbaatar2016commnet,das2019tarmac}. These methods usually assume low-dimensional observations, compact policies, or continuous communication. In contrast, LLM agents act over long textual histories and communication, making belief and partner modeling more challenging.

\section{Background}

\subsection{Decentralized LLM Collaboration}

In decentralized LLM collaboration, given a task, each agent is provided with a natural language prompt. The process can be described as a multi-turn interaction. Each agent maintains its own context locally and generates a response based on its own policy. The responses of all agents are aggregated within the environment into a joint action, after which the environment updates the task state based on this joint action and returns new observations, feedback, and rewards. When generating actions, each agent rely on its own local observations, historical context, received messages, and tools to which it has access. Agents can exchange information through the communication channels. The communication itself is part of the action. And decentralized execution supports parallel reasoning and reduces context pressure on individual agents.

\subsection{LLM Dec-POMDP Formulation}

We formulate decentralized LLM collaboration as a cooperative Dec-POMDP~\citep{oliehoek2016decpomdp,liu2025magrpo}, denoted by $\mathcal{M}=\langle \mathcal{I},\mathcal{V},C,M,\mathcal{S},\{\mathcal{O}_i\}_{i\in\mathcal{I}},\{\mathcal{A}_i\}_{i\in\mathcal{I}},R,T,H,\gamma\rangle$. Here, $\mathcal{I}=\{1,\ldots,n\}$ is the set of $n$ LLM agents, $\mathcal{V}$ is the token vocabulary, $C$ is the context-window length, $M$ is the maximum response length, $H$ is the episode horizon, and $\gamma\in[0,1]$ is the discount factor. The global state space is $\mathcal{S}=\mathcal{S}^{\mathrm{sys}}\times\mathcal{S}^{\mathrm{usr}}$, where $\mathcal{S}^{\mathrm{sys}}$ denotes the system-accessible state and $\mathcal{S}^{\mathrm{usr}}$ the user-hidden state. At turn $t$, the global state is $s_t=(s_t^{\mathrm{sys}},s_t^{\mathrm{usr}})\in\mathcal{S}$.

Each agent $i$ has a local observation space $\mathcal{O}_i$, with joint observation space $\mathcal{O}=\times_{i\in\mathcal{I}}\mathcal{O}_i$. At turn $t$, agent $i$ receives a local observation $o_{i,t}\in\mathcal{O}_i$, typically represented as a natural-language prompt, and the joint observation is $\mathbf{o}_t=(o_{i,t})_{i\in\mathcal{I}}\in\mathcal{O}$. Each agent also has an action space $\mathcal{A}_i$, with joint action space $\mathcal{A}=\times_{i\in\mathcal{I}}\mathcal{A}_i$. An action is a generated response sequence $a_{i,t}=(a_{i,t}^{1},a_{i,t}^{2},\ldots,a_{i,t}^{m_{i,t}})$, where $m_{i,t}\le M$ and thus $\mathcal{A}_i\subseteq\mathcal{V}^{\le M}$. The joint action is $\mathbf{a}_t=(a_{i,t})_{i\in\mathcal{I}}\in\mathcal{A}$. The environment evolves according to $T:\mathcal{S}\times\mathcal{A}\rightarrow\Delta(\mathcal{S})$, with $s_{t+1}\sim T(\cdot\mid s_t,\mathbf{a}_t)$. The joint reward is given by $R:\mathcal{S}^{\mathrm{sys}}\times\mathcal{A}\rightarrow\mathbb{R}$, with $r_t=R(s_t^{\mathrm{sys}},\mathbf{a}_t)$.

Agent $i$ acts from its local history $h_{i,t}=(o_{i,0},a_{i,0},o_{i,1},a_{i,1},\ldots,a_{i,t-1},o_{i,t})$, represented within a context window of at most $C$ tokens. The joint history is $\mathbf{h}_t=(h_{i,t})_{i\in\mathcal{I}}$. Each policy maps a local history to a distribution over response sequences, $\pi_{\theta_i}(a_{i,t}\mid h_{i,t})$, and the joint policy is $\boldsymbol{\pi}_{\boldsymbol{\theta}}=\{\pi_{\theta_i}\}_{i\in\mathcal{I}}$. Under decentralized execution, $\boldsymbol{\pi}_{\boldsymbol{\theta}}(\mathbf{a}_t\mid\mathbf{h}_t)=\prod_{i\in\mathcal{I}}\pi_{\theta_i}(a_{i,t}\mid h_{i,t})$. At the token level, $\pi_{\theta_i}(a_{i,t}\mid h_{i,t})=\prod_{\mu=1}^{m_{i,t}}\pi_{\theta_i}(a_{i,t}^{\mu}\mid h_{i,t},a_{i,t}^{<\mu})$.

The objective is to maximize the discounted expected return over the episode, $J(\boldsymbol{\theta})=\mathbb{E}_{\boldsymbol{\pi}_{\boldsymbol{\theta}}}[\sum_{t=0}^{H-1}\gamma^t r_t]$, yielding the optimal joint policy $\boldsymbol{\pi}^{*}=\arg\max_{\boldsymbol{\pi}}\mathbb{E}_{\boldsymbol{\pi}}[\sum_{t=0}^{H-1}\gamma^t r_t]$. An episode trajectory is $\tau=(\mathbf{o}_0,\mathbf{a}_0,r_0,\mathbf{o}_1,\mathbf{a}_1,r_1,\ldots,\mathbf{o}_{H-1},\mathbf{a}_{H-1},r_{H-1},\mathbf{o}_H)$. The discounted return from turn $t$ is $G_t=\sum_{\ell=t}^{H-1}\gamma^{\ell-t}r_\ell$, and the corresponding joint value on the joint history is $V^{\boldsymbol{\pi}}(\mathbf{h}_t)=\mathbb{E}_{\boldsymbol{\pi}}[\sum_{\ell=t}^{H-1}\gamma^{\ell-t}r_\ell\mid\mathbf{h}_t]$.

\subsection{MAAC}

In centralized-training decentralized-execution MAAC~\citep{liu2026collm}, each actor depends only on its local history at execution time, $\pi_{\theta_i}(a_{i,t}\mid h_{i,t})$. During training, a centralized critic estimates the team value. We write $V_{\phi}(\mathbf{h}_t)
=\mathbb{E}_{\boldsymbol{\pi}}[G_t\mid \mathbf{h}_t]$ to indicate that the critic is conditioned on the joint history.
The corresponding advantage estimate is $\hat A_t=r_t+\gamma V_{\phi}(\mathbf h_{t+1})-V_{\phi}(\mathbf h_t)$, and the actor objective uses the log-probability of the complete response sequence,

\begin{equation}
\mathcal{L}_{\pi}^{\mathrm{MAAC}}
=-\sum_t\sum_{i\in\mathcal{I}}\hat A_t
\log\pi_{\theta_i}(a_{i,t}\mid h_{i,t}).
\end{equation}

MAAC reduces the variance of Monte Carlo returns by bootstrapping from a learned value function and propagates terminal rewards back to intermediate turns. However, MAAC remains limited by the difficulty of extracting task-relevant state from long histories. It also lacks explicit modeling of the knowledge of other agents and future actions.

\section{Belief-Conditioned MAAC}
\label{sec:bcmaac}

\subsection{BCMAAC Formulation}

Each episode samples a task $\xi\sim\mathcal{D}$, which determines the environment $\mathrm{Env}$. Besides the language history $h_{i,t}$, each agent receives structured side information $x_{i,t}\in\mathcal{X}_i$ that can be constructed at execution time; the joint side information is $\mathbf x_t=(x_{i,t})_{i\in\mathcal I}$. The key constraint is that $x_{i,t}$ only depends on information available to agent $i$ \emph{at execution time}, including received communication. BCMAAC maintains a latent belief through a recurrent belief
encoder $f_\psi$, yielding
\begin{equation}
(c_{i,t},z_{i,t+1},\hat{\mathbf b}_{i,t},\hat{\mathbf p}_{i,t})
=f_{\psi}(x_{i,t},z_{i,t}).
\end{equation}
Here, $z_{i,t}$ is the recurrent state carried across turns, $c_{i,t}$ is the belief context passed to the actor and critic. And $\hat{\mathbf b}_{i,t}$ is a task belief prediction, and $\hat{\mathbf p}_{i,t}$ is a partner action prediction.

The BCMAAC actor is conditioned on local history and local belief context, $\pi_{\theta_i}(a_{i,t}\mid h_{i,t}, c_{i,t})$. For an LLM actor, the action remains a response sequence. The predictions $\hat{\mathbf b}_{i,t}$ and $\hat{\mathbf p}_{i,t}$ are train-time auxiliary heads, supervised by per-turn targets $\mathbf b^*_{i,t}$ and $\mathbf p^*_{i,t}$. These targets are produced by a labeling rule $\Gamma_{\xi}$. For $\Gamma_{\xi}$, it reads the completed trajectory $\tau$ together with the environment's internal state, and writes down the ground truth that the agents could not observe. Such privileged information is used only during training.

During training, the centralized critic aggregates all agent contexts. Let $\mathbf c_t=\mathrm{Joint}(\{c_{i,t}\}_{i\in\mathcal{I}}).$ $\mathrm{Joint}(\cdot)$ may be implemented by concatenation with fixed order. The critic estimates the joint value $V_{\phi}(\mathbf c_t) = \mathbb{E}_{\boldsymbol{\pi}}[G_t\mid \mathbf c_t].$

The BCMAAC advantage is $\hat A_t=r_t+\gamma V_{\phi}(\mathbf c_{t+1})-V_{\phi}(\mathbf c_t)$, and the actor and critic losses are
\begin{align}
\mathcal{L}_{\pi}
&=-\sum_t\sum_{i\in\mathcal{I}}
\hat A_t\log\pi_{\theta_i}(a_{i,t}\mid h_{i,t}, c_{i,t}),\\
\mathcal{L}_{V}
&=\sum_t\left(V_{\phi}(\mathbf c_t)-r_t-\gamma V_{\phi}(\mathbf c_{t+1})\right)^2.
\end{align}

\subsection{Theoretical Properties}

Let $\mathcal{F}_t$ denote the information available before sampling the joint action at turn $t$. The joint context $\mathbf{c}_t$ is computed as a compressed representation of the execution information in $\mathcal{F}_t$. Then $\mathbf{c}_t$ is $\mathcal{F}_t$-measurable.

Remark~\ref{rem:unbiased} shows that replacing $G_t$ with $G_t-\beta(\mathbf c_t)$ does not change the expected value of the policy gradient estimate; the form is entirely analogous to the score-function baseline identity ~\citep{williams1992reinforce}.

\begin{remark}
\label{rem:unbiased}
Let $\mathbf c_t$ be fixed before sampling the current joint action $\mathbf a_t$, and let $\beta(\mathbf c_t)$ be independent of the current action. Then replacing $G_t$ with $G_t-\beta(\mathbf c_t)$ does not change the expectation of the policy-gradient estimator:
\begin{align*}
&\mathbb{E}_{\boldsymbol{\pi}}\left[
\sum_{i\in\mathcal{I}}\nabla_{\theta_i}
\log\pi_{\theta_i}(a_{i,t}\mid h_{i,t}, c_{i,t})
\left(G_t-\beta(\mathbf c_t)\right)\right]\\
&\quad =
\mathbb{E}_{\boldsymbol{\pi}}\left[
\sum_{i\in\mathcal{I}}\nabla_{\theta_i}
\log\pi_{\theta_i}(a_{i,t}\mid h_{i,t}, c_{i,t})G_t\right].
\end{align*}
\end{remark}

\begin{proof}
Conditioning on $\mathcal{F}_t$, $\beta(\mathbf c_t)$ is pre-action and does not depend on $a_{i,t}$. Since the joint policy factorizes, marginalizing over $\mathbf a_{-i,t}$ reduces each term to a single-agent identity. Using
$\pi_{\theta_i}\nabla_{\theta_i}\log\pi_{\theta_i}=\nabla_{\theta_i}\pi_{\theta_i}$,
\begin{align*}
&\mathbb{E}_{a_{i,t}\sim\pi_{\theta_i}}\left[
\nabla_{\theta_i}\log\pi_{\theta_i}(a_{i,t}\mid h_{i,t}, c_{i,t})
\beta(\mathbf c_t)\mid\mathcal{F}_t\right]\\
&\quad =
\beta(\mathbf c_t)\nabla_{\theta_i}
\sum_{a_{i,t}}\pi_{\theta_i}(a_{i,t}\mid h_{i,t}, c_{i,t})=0.
\end{align*}
Therefore the baseline term contributes zero in expectation. By linearity and then taking expectation over $\mathcal{F}_t$, the two policy-gradient estimators have the same expectation.
\end{proof}

Proposition~\ref{prop:sufficiency} shows the ideal case: when the learned context is sufficient for future return, it loses
no value-relevant information compared to the full-history context (proof in Appendix~\ref{app:proofs}).

\begin{proposition}
\label{prop:sufficiency}
If the joint belief context $\mathbf c_t$ is sufficient for future return in the sense that
\begin{equation*}
\mathbb{E}_{\boldsymbol{\pi}}[G_t\mid \mathcal{F}_t,\mathbf c_t]
=\mathbb{E}_{\boldsymbol{\pi}}[G_t\mid \mathbf c_t],
\end{equation*}
then the value conditioned on all available information equals the value conditioned on context: $V^{\boldsymbol{\pi}}(\mathcal{F}_t)=V^{\boldsymbol{\pi}}(\mathbf c_t).$

\end{proposition}

Remark~\ref{rem:aux} notes that the auxiliary losses supply the belief encoder with a per-turn representation signal, which is especially useful when task rewards are sparse.

\begin{remark}
\label{rem:aux}
In sparse-reward tasks where task reward is concentrated near the end of the episode, return-driven representation learning depends on discounted returns. In BCMAAC, the belief encoder also receives gradients from
\[
\mathcal{L}_{\mathrm{aux}}
=\sum_t\sum_{i\in\mathcal{I}}\big[\lambda_b\ell_b(\hat{\mathbf b}_{i,t},\mathbf b^*_{i,t})
+\lambda_p\ell_p(\hat{\mathbf p}_{i,t},\mathbf p^*_{i,t})\big],
\]
where the labels $\mathbf b^*_{i,t}$ and $\mathbf p^*_{i,t}$ are generated from the trajectory by the labeling rule $\Gamma_{\xi}(\tau)$. These auxiliary gradients compare predictions directly with per-turn labels. Thus, whenever such labels are available, the belief encoder receives an additional per-turn representation signal, indicating that BCMAAC supplies an additional source of representation learning for state estimation and the modeling of agents.
\end{remark}

\subsection{BCMAAC Algorithm}

\begin{wrapfigure}{r}{0.55\textwidth}
  \centering
  \vspace{-6mm}
  \includegraphics[width=0.9\linewidth]{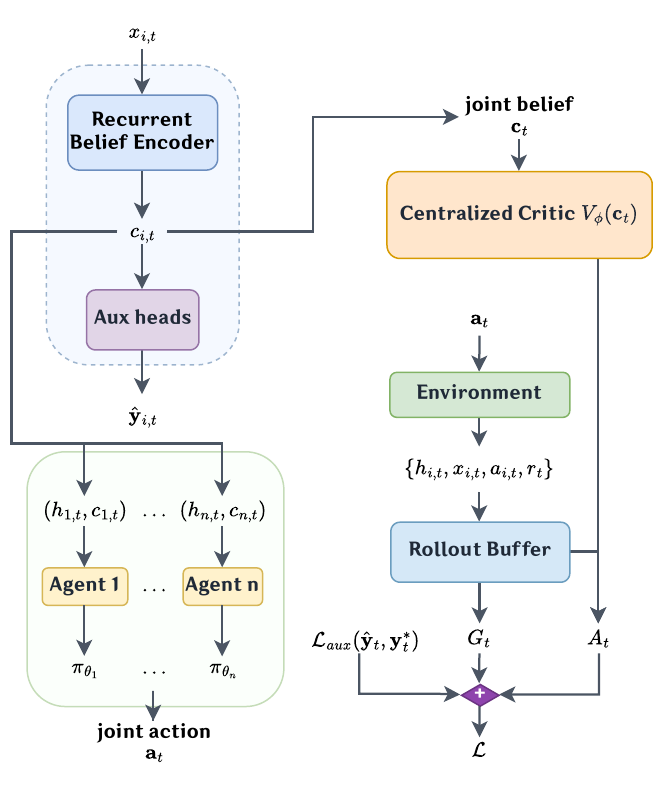}
  \vspace{-4mm}
  \caption{Illustration of the BCMAAC framework.\vspace{-8mm}}
  \label{fig:bcmaac-frame}
\end{wrapfigure}

Figure~\ref{fig:bcmaac-frame} and Algorithm~\ref{alg:bcmaac} summarize the full training procedure, which has three components: trajectory collection, auxiliary label generation, and actor-critic updates. After each rollout, one-step TD targets are computed, and auxiliary labels are generated by $\{\mathbf y^*_{i,t}\}_{i,t}\leftarrow\Gamma_{\xi}(\tau)$, where $\mathbf y^*_{i,t}=(\mathbf b^*_{i,t},\mathbf p^*_{i,t})$ collects the task-belief and partner-action targets that supervise the auxiliary predictions $\hat{\mathbf y}_{i,t}=(\hat{\mathbf b}_{i,t},\hat{\mathbf p}_{i,t})$.

For auxiliary supervision, $\ell_b$ is a masked BCE loss for task-belief prediction and $\ell_p$ a cross-entropy loss for partner-action prediction, with the mask $\boldsymbol{\omega}_{i,t}$ excluding unlabeled task-belief entries. Summing over agents and turns gives $\mathcal{L}_b=\sum_t\sum_{i\in\mathcal{I}}\ell_b(\hat{\mathbf b}_{i,t},\mathbf b^*_{i,t})$ and $\mathcal{L}_p=\sum_t\sum_{i\in\mathcal{I}}\ell_p(\hat{\mathbf p}_{i,t},\mathbf p^*_{i,t})$, and the auxiliary objective is $\mathcal{L}_{\mathrm{aux}}=\lambda_b\mathcal{L}_b+\lambda_p\mathcal{L}_p$.

During training, contexts are recomputed from the stored structured side information, $(c^b_{i,t},z^b_{i,t+1},\hat{\mathbf b}^b_{i,t},\hat{\mathbf p}^b_{i,t})\leftarrow f_{\psi}(x^b_{i,t},z^b_{i,t})$, and $\mathbf c^b_t=\mathrm{Joint}(\{c^b_{i,t}\}_{i\in\mathcal{I}})$. Given a rollout buffer $\mathcal B$, BCMAAC computes one-step TD advantages. The critic is optimized by TD regression, the actors by advantage-weighted response log-probabilities, and the belief encoder by the combined value, policy, and auxiliary losses. Thus, with a value-loss weight $\lambda_V$, $\phi$ is updated only by the value loss, $\theta_i$ by the policy loss, and $\psi$ by $\lambda_V\mathcal L_V+\mathcal L_{\mathrm{aux}}+\mathcal L_\pi$.

During execution, the centralized critic and auxiliary heads are not required. Each agent recursively computes $(c_{i,t},z_{i,t+1})\leftarrow f_{\psi}(x_{i,t},z_{i,t})$ and samples $a_{i,t}\sim\pi_{\theta_i}(\cdot\mid h_{i,t}, c_{i,t})$. Since $x_{i,t}$ is constructed only from agent $i$'s execution-time information, $c_{i,t}$ is local as well. Thus BCMAAC preserves decentralized execution: $\boldsymbol{\pi}_{\boldsymbol{\theta}}(\mathbf a_t\mid \mathbf h_t,\mathbf c_t)
=\prod_{i\in\mathcal{I}}\pi_{\theta_i}(a_{i,t}\mid h_{i,t}, c_{i,t}).$

\begin{algorithm}[t]
\caption{BCMAAC}
\label{alg:bcmaac}
\begin{algorithmic}[1]
\STATE {\bfseries Input:} Task distribution $\mathcal{D}$; initial parameters $(\boldsymbol{\theta},\psi,\phi)$;
\FOR{epoch $e=1,\ldots,E$}
    \STATE Initialize rollout buffer $\mathcal{B}\leftarrow\emptyset$
    \FOR{rollout $k=1,\ldots,K_{\mathrm{roll}}$}
        \STATE Sample $\xi\sim\mathcal{D}$; reset $\xi$; initialize $h_{i,0},x_{i,0},z_{i,0}$ and trajectory $\tau\leftarrow\emptyset$
        \FOR{turn $t=0,\ldots,H-1$}
            \STATE Receive current prompt $o_{i,t}$ and structured observation $x_{i,t}$; update local history $h_{i,t}$
            \STATE $(c_{i,t},z_{i,t+1},\hat{\mathbf b}_{i,t},\hat{\mathbf p}_{i,t})\leftarrow f_{\psi}(x_{i,t},z_{i,t})$, $\forall i$
            \STATE Sample $a_{i,t}\sim\pi_{\theta_i}(\cdot\mid h_{i,t},c_{i,t})$, $\forall i$; execute $\mathbf a_t$ in $\mathrm{Env}$
            \STATE Observe $(\mathbf o_{t+1},\mathbf x_{t+1},r_t)$; append transition to $\tau$
        \ENDFOR
        \STATE Compute one-step TD targets; generate labels $\{(\mathbf b^*_{i,t},\mathbf p^*_{i,t})\}_{i,t}\leftarrow\Gamma_{\xi}(\tau)$; add $\tau$ to $\mathcal{B}$
    \ENDFOR
    \STATE Initialize losses $\bar{\mathcal{L}}_V,\bar{\mathcal{L}}_\pi,\bar{\mathcal{L}}_{\mathrm{aux}}\leftarrow0$
    \FOR{trajectory $\tau^b\in\mathcal{B}$}
        \STATE Reset recurrent states $z^b_{i,0}\leftarrow0$, $\forall i$
        \FOR{transition $t$ in $\tau^b$}
            \STATE Recompute local contexts and auxiliary predictions via $f_\psi$, $\forall i$
            \STATE Form joint context $\mathbf c^b_t$ and advantage $\hat A^b_t$
            \STATE Accumulate value, policy, and auxiliary losses

        \ENDFOR
    \ENDFOR
    \STATE Update critic $\phi\leftarrow\phi-\alpha_V\nabla_\phi(\lambda_V\bar{\mathcal{L}}_V)$
    \STATE Update actors $\theta_i\leftarrow\theta_i-\alpha_\pi\nabla_{\theta_i}\bar{\mathcal{L}}_\pi$, $\forall i\in\mathcal{I}$
    \STATE Update encoder $\psi\leftarrow\psi-\alpha_\psi\nabla_\psi(\lambda_V\bar{\mathcal{L}}_V+\bar{\mathcal{L}}_\pi+\bar{\mathcal{L}}_{\mathrm{aux}})$
\ENDFOR
\STATE {\bfseries Output:} Belief encoder $f_{\psi}$ and joint policy $\boldsymbol{\pi}_{\boldsymbol{\theta}}$
\end{algorithmic}
\end{algorithm}
\vspace{-4mm}

\begin{table}[t]
\centering
\footnotesize
\setlength{\tabcolsep}{4pt}
\caption{Evaluation results across Minecraft, Web, and Tool collaboration. AC denotes single-agent actor-critic method; Parallel and Pipeline represents 2 test-time prompt-based multi-agent collaboration methods; MAAC is the centralized-critic multi-agent actor-critic baseline; BCMAAC-P and BCMAAC-B are variants that only partner-action and task-belief auxiliary beliefs, respectively. Success is the success rate and match score is the normalized answer-matching score, both are reported as percentages and averaged over evaluation runs. Standard deviations are reported in subscripts.\vspace{2mm}}
\begin{tabular}{lcccccc}
\toprule
 & \multicolumn{2}{c}{\textbf{Minecraft}}
 & \multicolumn{2}{c}{\textbf{Web}}
 & \multicolumn{2}{c}{\textbf{Tool}} \\
\cmidrule(lr){2-3}
\cmidrule(lr){4-5}
\cmidrule(lr){6-7}

\multirow{3}{*}{\textbf{Method}}
 & Resource & Bridge & Cross-Lingual & Multi-Source & Stateful & Filesystem \\
 & Gathering & Building & Gathering & Gathering & Tool & Wrangling \\
 & Success (\%) & Success (\%) & Match Score (\%) & Match Score (\%)
 & Success (\%) & Success (\%) \\
\midrule

Raw Model
 & \(0.0_{\pm 0.0}\)
 & \(0.0_{\pm 0.0}\)
 & \(45.0_{\pm 0.0}\)
 & \(11.3_{\pm 0.0}\)
 & \(33.3_{\pm 0.0}\)
 & \(53.8_{\pm 0.0}\) \\

AC
 & \(70.1_{\pm 5.7}\)
 & \(30.1_{\pm 7.2}\)
 & \(51.9_{\pm 3.4}\)
 & \(19.7_{\pm 3.8}\)
 & \(55.8_{\pm 7.5}\)
 & \(66.9_{\pm 8.0}\) \\

\midrule

Parallel
 & \(0.0_{\pm 0.0}\)
 & \(0.0_{\pm 0.0}\)
 & \(37.3_{\pm 1.0}\)
 & \(16.9_{\pm 2.7}\)
 & \(33.0_{\pm 5.6}\)
 & \(25.3_{\pm 0.0}\) \\

Pipeline
 & \(0.0_{\pm 0.0}\)
 & \(0.0_{\pm 0.0}\)
 & \(35.8_{\pm 1.6}\)
 & \(18.0_{\pm 0.0}\)
 & \(26.7_{\pm 0.0}\)
 & \(4.6_{\pm 4.2}\) \\

\midrule

MAAC
 & \(78.4_{\pm 8.9}\)
 & \(80.6_{\pm 10.3}\)
 & \(45.1_{\pm 3.9}\)
 & \(47.6_{\pm 2.8}\)
 & \(59.2_{\pm 6.1}\)
 & \(58.9_{\pm 8.0}\) \\

BCMAAC-P
 & \(83.7_{\pm 5.3}\)
 & \(89.3_{\pm 9.8}\)
 & \(50.9_{\pm 0.9}\)
 & \(49.4_{\pm 3.5}\)
 & \(64.5_{\pm 11.2}\)
 & \(64.6_{\pm 2.6}\) \\

BCMAAC-B
 & \(88.9_{\pm 8.5}\)
 & \(84.8_{\pm 13.6}\)
 & \(51.6_{\pm 6.5}\)
 & \(53.3_{\pm 6.4}\)
 & \(70.4_{\pm 6.3}\)
 & \(71.2_{\pm 5.0}\) \\

BCMAAC
 & \(\mathbf{92.8}_{\pm 3.7}\)
 & \(\mathbf{94.4}_{\pm 9.6}\)
 & \(\mathbf{53.2}_{\pm 5.1}\)
 & \(\mathbf{58.5}_{\pm 2.3}\)
 & \(\mathbf{77.5}_{\pm 6.4}\)
 & \(\mathbf{73.0}_{\pm 1.9}\) \\

\bottomrule
\end{tabular}
\label{tab:main-eval}
\end{table}

\section{Experiments}
\label{sec:experiments}

\subsection{Setup}
\label{subsec:exp-setup}

\paragraph{Minecraft Collaboration.}
Minecraft games (Figure~\ref{fig:minecraft-envs}) are long-horizon testbeds with explicit physical state and partial observability. Each agent observes only its local view and received messages, where the agents maintain environment beliefs, track the other agents' progress, avoid redundant exploration, and coordinate actions through communication. 
In \textit{Bridge Building}, two agents start on opposite sides of a damaged gap and should construct a safe connected path within limited turns. The gap contains visually ambiguous pillars that may either provide useful support or indicate hazardous regions, where nearby construction incurs large penalties. Because pillar types are initially hidden, agents must explore informative viewpoints, communicate discoveries, and coordinate bridge segments from both sides.
In \textit{Resource Gathering}, resource cells contain different amounts of wood, stone, and iron, and agents must satisfy all collection targets within a fixed horizon. Agents have asymmetric extraction capabilities, so they may observe resources that only the other agents' can collect. Effective collaboration therefore requires separating actionable observations from information that should be communicated. At each turn, agents specify broadcast resource facts, movement paths, and extraction targets, and extraction is only allowed near the final location.

\paragraph{Web Collaboration.}

\begin{wrapfigure}{R}{0.5\textwidth}
  \centering
  \vspace{-8mm}
  \begin{subfigure}[b]{0.46\linewidth}\centering\includegraphics[width=\textwidth]{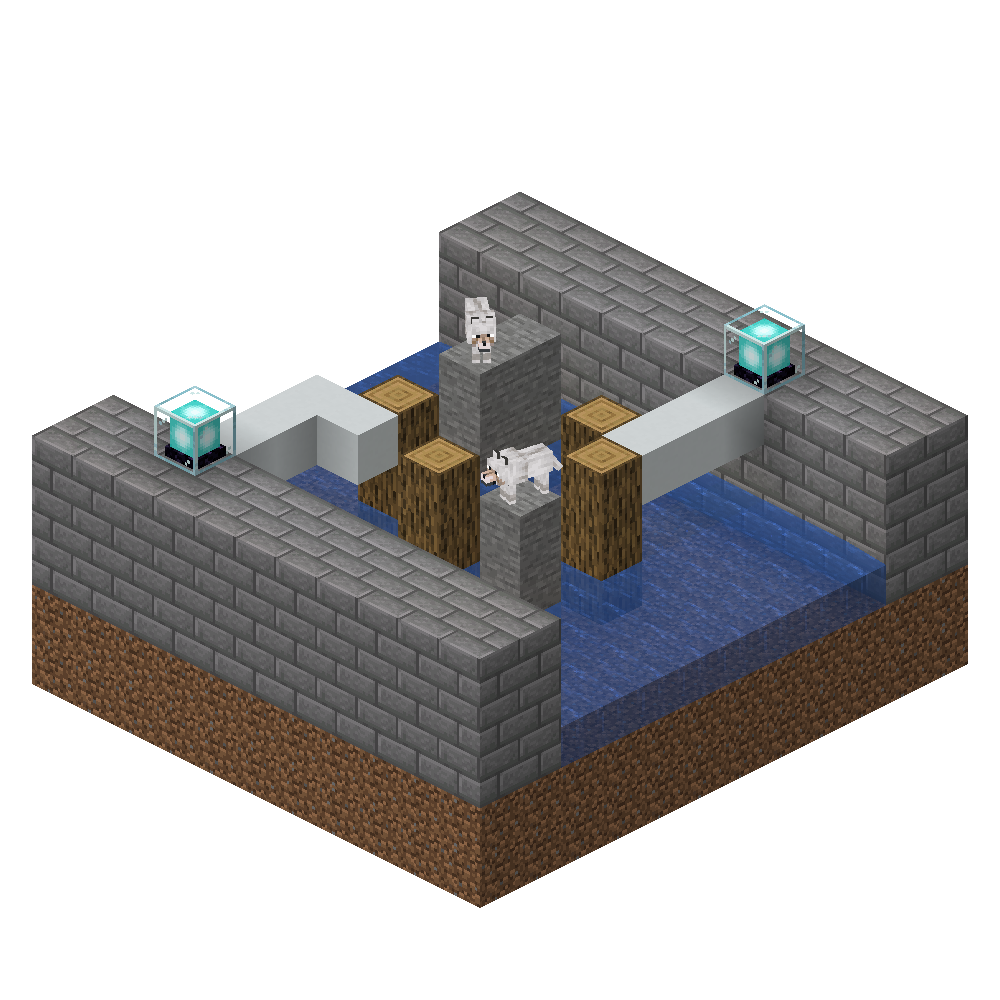}\caption{\textit{Bridge Building}}\label{subfig:env-bridge}\end{subfigure}
  \hfill
  \begin{subfigure}[b]{0.46\linewidth}\centering\includegraphics[width=\textwidth]{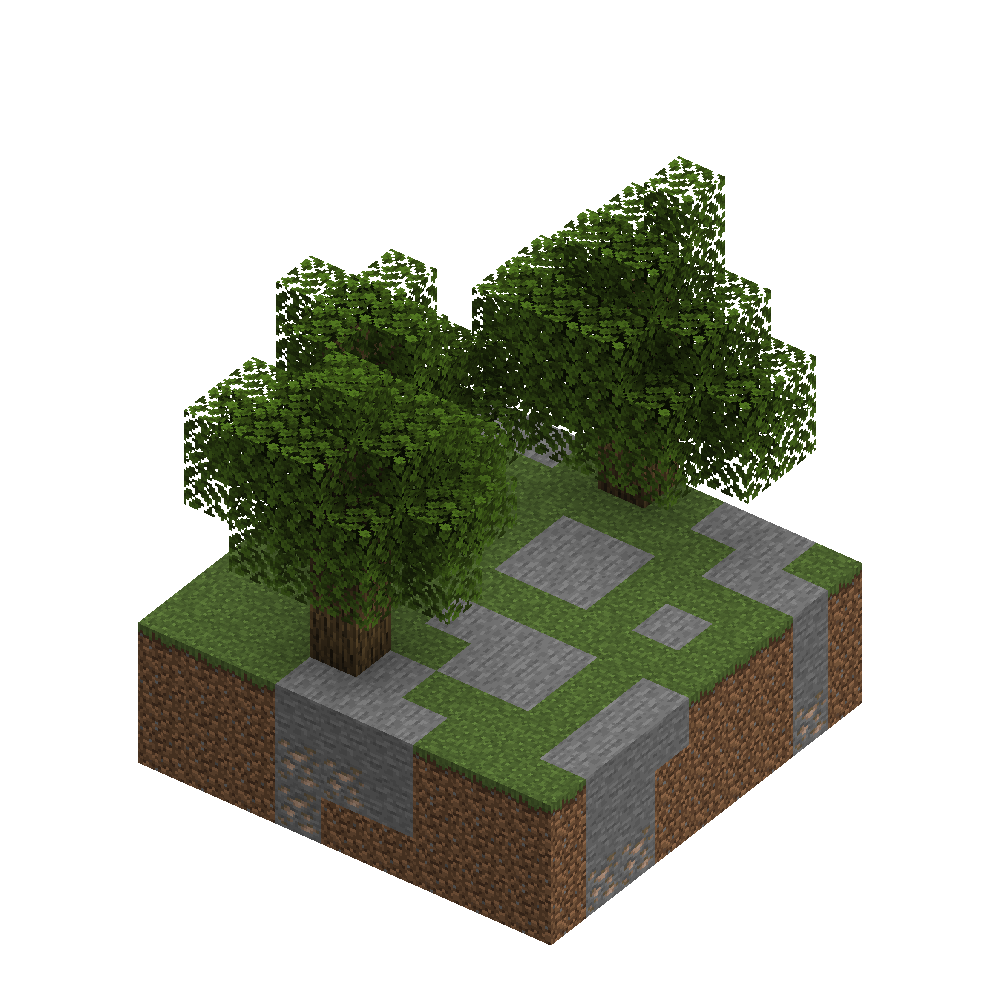}\caption{\textit{Resource Gathering}}\label{subfig:env-rg}\end{subfigure}
  \caption{Minecraft tasks. (a) In \textit{Bridge Building}, agents start at the blue beacon endpoints and must construct a white-block bridge under limited vision. Brown wooden pillars are safe, whereas gray stone pillars with wolves are dangerous decoys. (b) In \textit{Resource Gathering}, agents with complementary tools observe only local regions and  collect from trees, stone, and minerals of iron.\vspace{-1mm}}
  \label{fig:minecraft-envs}
\end{wrapfigure}
In Web Collaboration, agents operate in an open-ended environment. The relevant state is not a location or block layout, but the progress of the search process: what evidence each agent has collected, which sources still look promising, whether the available evidence is enough to support an answer, and what the other agent are likely to do next. As the dominant reward is given after the final answer is checked, agents receive little direct feedback while searching. They should decide whether to continue gathering evidence, share partial findings, wait for more information from the other agents, or commit to an answer.
In \textit{Cross-Lingual Gathering}, agents gather cross-lingual information sources~\citep{longpre2021mkqa} and produce a final answer. Agents should bridge language and phrasing differences, convert locally retrieved information into evidence useful to the other agents, and decide whether evidence from different languages supports the same answer. Simple parallel search may improve coverage, but without modeling what the other agent have found, agents can still duplicate search effort or submit incomplete answers too early, leading to inefficiencies in the collaboration process.
In \textit{Multi-Source Gathering}, agents aggregate evidence from multiple sources~\citep{chen2020hybridqa}. A single source may provide only partial evidence, multiple sources can be redundant or conflicting, and the final answer depends on combining evidence across sources. Effective collaboration requires agents to know not only what they have found, but also what the other agents have covered, what remains missing, and whether the agents are ready to answer.

\paragraph{Tool Collaboration.}

The tool-using task, one agent inspect the environment, while the other can modify it. Information acquisition and execution must occur in separate turns. The inspecting agent must retrieve and communicate the necessary facts, after which the acting agent can execute.
In \textit{Stateful Tool}, we adapt Apple ToolSandbox~\citep{lu2025toolsandbox} into a two-agent setting, where a phone environment contains contacts, messages, reminders, and device settings. Two agents are conducting tasks such as retrieving the latest message, disabling cellular data, or rescheduling a reminder, are provided in natural language. The state inspector has read-only access, and the tool executor can modify the environment, send messages, and submit the final answer. Arguments such as phone numbers, contact IDs, and timestamps can be retrieved by the inspector and communicated in a prior turn before the executor use them. Each agent produces a structured JSON action specifying broadcast facts, information requests, the intended next step, a readiness indicator, at most one tool call, and an optional final answer.
\textit{Filesystem Wrangling} is adapted from the multi-turn \texttt{GorillaFileSystem} task in BFCL V3~\citep{patil2025bfcl}. It uses a deterministic in-memory filesystem that can be regarded as a miniature MCP server accessed by two clients with restricted permissions. The explorer has read-only tools, including \texttt{ls}, \texttt{find}, \texttt{cat}, and \texttt{grep}, while the editor has write tools such as \texttt{mkdir}, \texttt{touch}, \texttt{echo}, \texttt{mv}, \texttt{cp}, and \texttt{rm}. Since the editor cannot inspect the filesystem, the explorer must communicate all relevant paths. Hard operations such as \texttt{mv}, \texttt{cp}, \texttt{rm}, and \texttt{rmdir} are permitted only when the source path has been found in a previous turn; also, paths discovered in the current turn cannot be consumed immediately by the editor.

\begin{figure}[t]
\centering
\begin{subfigure}[b]{0.32\textwidth}\centering\includegraphics[width=\textwidth]{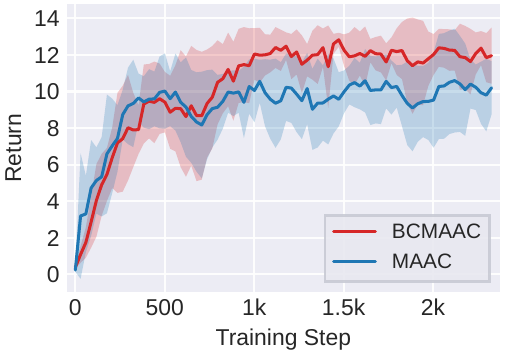}\caption{Minecraft $\mid$ \textit{Resource Gathering}}\label{subfig:rg}\end{subfigure}\hfill
\begin{subfigure}[b]{0.32\textwidth}\centering\includegraphics[width=\textwidth]{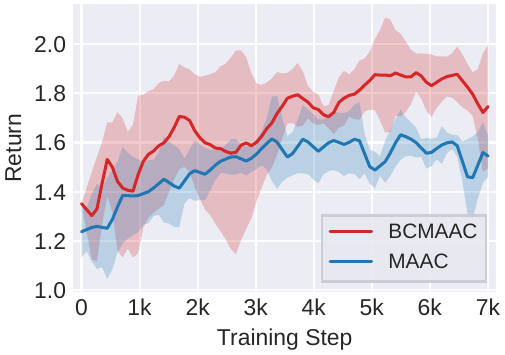}\caption{Web $\mid$ \textit{Cross-Lingual Gathering}}\label{subfig:clg}\end{subfigure}\hfill
\begin{subfigure}[b]{0.32\textwidth}\centering\includegraphics[width=\textwidth]{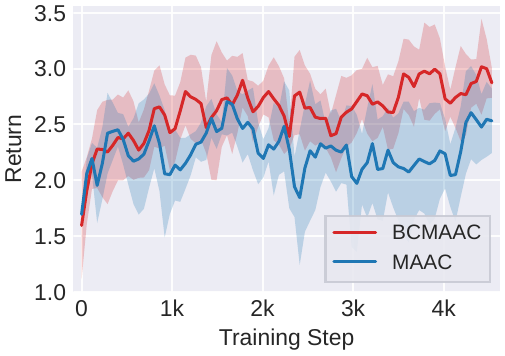}\caption{Tool $\mid$ \textit{Stateful Tool}}\label{subfig:st}\end{subfigure}

\begin{subfigure}[b]{0.32\textwidth}\centering\includegraphics[width=\textwidth]{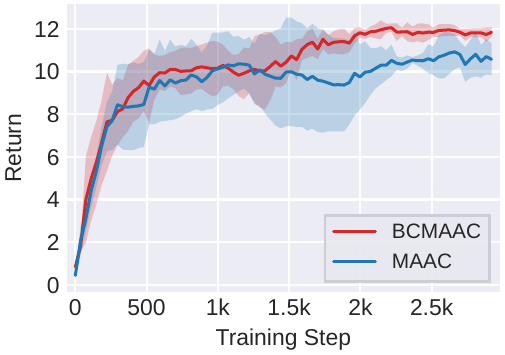}\caption{Minecraft $\mid$ \textit{Bridge Building}}\label{subfig:bb}\end{subfigure}\hfill
\begin{subfigure}[b]{0.32\textwidth}\centering\includegraphics[width=\textwidth]{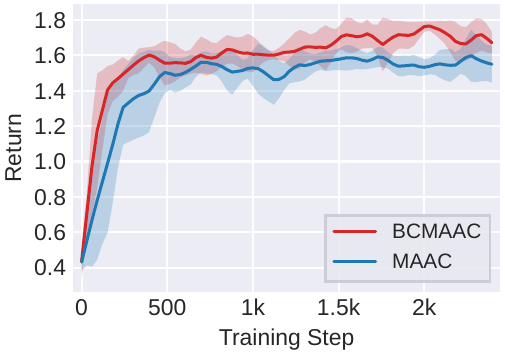}\caption{Web $\mid$ \textit{Multi-Source Gathering}}\label{subfig:msg}\end{subfigure}\hfill
\begin{subfigure}[b]{0.32\textwidth}\centering\includegraphics[width=\textwidth]{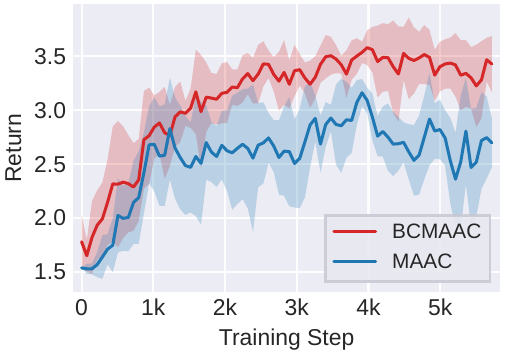}
\caption{Tool $\mid$ \textit{Filesystem Wrangling}}\label{subfig:fw}\end{subfigure} \caption{A comparison between BCMAAC (red) and MAAC (blue) across the six tasks. The y-axis denotes the expected return, and the x-axis denotes the environment steps. Results are all averaged over five runs, and shaded regions indicate the standard deviation. Curves are smoothed using a time-weighted exponential moving average.}
\label{fig:main-results}
\end{figure}

\subsection{Results}
\label{subsec:main-results}

Table \ref{tab:main-eval} reports the main evaluation results. \textbf{BCMAAC achieves the best performance on all six tasks.} Trained multi-agent methods substantially outperform single-agent AC on tasks that can be effectively decomposed across agents, particularly \textit{Resource Gathering}, \textit{Bridge Building}, \textit{Multi-Source Gathering}, and \textit{Stateful Tool}, indicating that these tasks are not merely tests of individual model capabilities, but require task decomposition, complementary information, and cross-agent collaboration. In contrast, collaboration based solely on prompts (Parallel, Pipeline) fails to consistently improve performance, \textbf{indicating that simply adding agents without learning effective collaboration methods is insufficient}. For standard deviations, BCMAAC exhibits comparable or lower variability than MAAC overall, indicating that it can improve stability; BCMAAC-P and BCMAAC-B show higher standard deviations than full BCMAAC on many tasks, indicating that BCMAAC with two auxiliary objectives not only improve average performance but provide a complementary regularization effect.

Tasks with low parallelism are an exception, the most typical being \textit{Filesystem Wrangling}: the Editor can only act on paths that the Explorer has already found and communicated in the previous turn, making the two roles more akin to a sequential pipeline than a parallel division of labor. Consequently, a single-agent AC (66.9\%) that handles both reading and writing and eliminates the need for cross-agent handoffs is inherently strong, even outperforming the decentralized MAAC (58.9\%). Even so, BCMAAC can still track the Editor's readiness status, outperforming the single-agent approach with 73.0\%, indicating that what truly matters is effective collaboration.

Compared to MAAC, BCMAAC's largest gain occurs in \textit{Stateful Tool} (59.2\%$\to$77.5\%, $+18.3$), followed by \textit{Resource Gathering} (78.4\%$\to$92.8\%) and \textit{Filesystem Wrangling} (58.9\%$\to$73.0\%). Both \textit{Stateful Tool} and \textit{Resource Gathering} require maintenance of hidden state, the former must track which facts the executor still lacks and when to submit them, while the latter must remember resource locations and coordinate with agents of different capabilities. This is exactly the setting where belief context helps. While MAAC already performs strongly on \textit{Bridge Building}, yet BCMAAC still improves from 80.6\% to 94.4\%, indicating that structured belief remains beneficial for collaboration quality. On Web tasks, BCMAAC also consistently outperforms MAAC. \textit{Cross-Lingual Gathering} shows moderate gains (45.1\%$\to$53.2\%), while \textit{Multi-Source Gathering} is even greater (47.6\%$\to$58.5\%), consistent with the task structure, as multi-source question-answering relies more on complementary evidence.

\textbf{Our ablation studies demonstrate that the two auxiliary signals are complementary.} BCMAAC-P and BCMAAC-B perform between MAAC and the full BCMAAC on all tasks, indicating that both partner-action prediction and task-belief supervision can improve upon the original MAAC critic. BCMAAC-B performs better on \textit{Resource Gathering}, \textit{Cross-Lingual Gathering}, \textit{Multi-Source Gathering}, and the two Tool tasks, all of which center on maintaining the state of resources/evidence/tools and the other agents' readiness; BCMAAC-P performs better on \textit{Bridge Building}, where coordinating the other agents' actions and dividing construction tasks is particularly important. The full BCMAAC achieves the best overall performance, demonstrating that combining the modeling of task beliefs with partner actions yields the best collaboration.

\subsection{Analysis}\label{subsec}

Figure~\ref{fig:main-results} shows the performance on six tasks. Overall, \textbf{BCMAAC converges faster, achieves higher rewards, and is more stable than MAAC on most tasks,} as the structured belief context improves advantage estimation.

In \textit{Filesystem Wrangling}, BCMAAC overtakes MAAC and maintains a markedly higher, less volatile reward. This is consistent with better tracking of Editor readiness and more accurate timing of hard actions. In \textit{Resource Gathering}, BCMAAC likewise improves more smoothly, suggesting that the belief encoder captures resource states, remaining objectives, and the other agents' capabilities useful for intermediate decisions. In \textit{Bridge Building} and \textit{Stateful Tool}, the methods are initially similar, but BCMAAC converges to a higher final reward. Gains are smaller on the Web tasks: BCMAAC maintains a modest advantage in \textit{Multi-Source Gathering}, while \textit{Cross-Lingual Gathering} remains noisy and shows improvement at end. The task-belief and partner-action objectives help track retrieved evidence and answer readiness, reducing redundant searches and premature responses.

Results suggest that BCMAAC benefits from latent collaboration state explicitly rather than reconstructing it solely from raw histories. Its centralized critic estimates returns from structured joint beliefs, while the auxiliary objectives provide dense turn-level supervision when communication and exploration produces little immediate reward.

In nearly fully observable tasks, or when short histories are sufficient, the advantage of BCMAAC may be limited. Its effectiveness depends on whether the reduction in estimation error outweighs the information lost through compression. In the partially observable, long-horizon tasks studied in this paper, this trade-off falls clearly on BCMAAC's side.

\section{Conclusion}

We study decentralized LLM collaboration and introduce BCMAAC. BCMAAC maintains structured belief context during execution, performs centralized value estimation from joint belief context during training, and preserves decentralized decision-making at inference time. BCMAAC consistently outperforms MAAC and improves training stability across Minecraft, Web, and tool-based collaboration tasks. The gains are strongest when agents must track latent state, integrate the information of other agents, coordinate delayed actions, and combine evidence from multiple sources. These results establish structured belief conditioning as an effective mechanism for decentralized LLM collaboration.

Our study has a few limitations. Although our framework is applicable to multiple agents, all six tasks are two-agent tasks, as the characteristics of the tasks we designed are best suited for pairing the two agents into complementary roles. For example, in \textit{Filesystem Wrangling}, the toolset splits cleanly into read-only commands and write commands, so an explorer with read-only permissions and an editor with write permissions can be paired accordingly.


\bibliographystyle{plainnat}
\bibliography{main}

\begin{thebibliography}{29}
\providecommand{\natexlab}[1]{#1}
\providecommand{\url}[1]{\texttt{#1}}
\expandafter\ifx\csname urlstyle\endcsname\relax
  \providecommand{\doi}[1]{doi: #1}\else
  \providecommand{\doi}{doi: \begingroup \urlstyle{rm}\Url}\fi

\bibitem[Chen et~al.(2020)Chen, Zha, Chen, Xiong, Wang, and Wang]{chen2020hybridqa}
Wenhu Chen, Hanwen Zha, Zhiyu Chen, Wenhan Xiong, Hong Wang, and William~Yang Wang.
\newblock {HybridQA}: A dataset of multi-hop question answering over tabular and textual data.
\newblock In \emph{Findings of the Association for Computational Linguistics: EMNLP 2020}, pages 1026--1036, Online, 2020. Association for Computational Linguistics.
\newblock \doi{10.18653/v1/2020.findings-emnlp.91}.
\newblock URL \url{https://aclanthology.org/2020.findings-emnlp.91/}.

\bibitem[Das et~al.(2019)Das, Gervet, Romoff, Batra, Parikh, Rabbat, and Pineau]{das2019tarmac}
Abhishek Das, Th{\'e}ophile Gervet, Joshua Romoff, Dhruv Batra, Devi Parikh, Mike Rabbat, and Joelle Pineau.
\newblock {TarMAC}: Targeted multi-agent communication.
\newblock In \emph{Proceedings of the 36th International Conference on Machine Learning}, volume~97 of \emph{Proceedings of Machine Learning Research}, pages 1538--1546. PMLR, 2019.

\bibitem[Du et~al.(2024)Du, Li, Torralba, Tenenbaum, and Mordatch]{du2024debate}
Yilun Du, Shuang Li, Antonio Torralba, Joshua~B. Tenenbaum, and Igor Mordatch.
\newblock Improving factuality and reasoning in language models through multiagent debate.
\newblock In \emph{Proceedings of the 41st International Conference on Machine Learning}, volume 235 of \emph{Proceedings of Machine Learning Research}, pages 11733--11763. PMLR, 2024.
\newblock URL \url{https://proceedings.mlr.press/v235/du24e.html}.

\bibitem[Feng et~al.(2026)Feng, Zheng, He, Zhang, and An]{feng2026drmas}
Lang Feng, Longtao Zheng, Shuo He, Fuxiang Zhang, and Bo~An.
\newblock {Dr. MAS}: Stable reinforcement learning for multi-agent {LLM} systems.
\newblock \emph{arXiv preprint arXiv:2602.08847}, 2026.
\newblock \doi{10.48550/arXiv.2602.08847}.
\newblock URL \url{https://arxiv.org/abs/2602.08847}.

\bibitem[Foerster et~al.(2016)Foerster, Assael, de~Freitas, and Whiteson]{foerster2016dial}
Jakob~N. Foerster, Ioannis~Alexandros Assael, Nando de~Freitas, and Shimon Whiteson.
\newblock Learning to communicate with deep multi-agent reinforcement learning.
\newblock In \emph{Advances in Neural Information Processing Systems}, volume~29, 2016.

\bibitem[Guo et~al.(2025)Guo, Yang, Zhang, et~al.]{guo2025deepseekr1}
Daya Guo, Dejian Yang, Haowei Zhang, et~al.
\newblock {DeepSeek-R1} incentivizes reasoning in {LLMs} through reinforcement learning.
\newblock \emph{Nature}, 645:\penalty0 633--638, 2025.
\newblock \doi{10.1038/s41586-025-09422-z}.

\bibitem[He et~al.(2016)He, Boyd-Graber, Kwok, and Daum{\'e}]{he2016deepopp}
He~He, Jordan Boyd-Graber, Kevin Kwok, and Hal Daum{\'e}, III.
\newblock Opponent modeling in deep reinforcement learning.
\newblock In \emph{Proceedings of the 33rd International Conference on Machine Learning}, volume~48 of \emph{Proceedings of Machine Learning Research}, pages 1804--1813. PMLR, 2016.

\bibitem[{Kimi Team} et~al.(2025){Kimi Team}, Du, Gao, et~al.]{kimiteam2025kimi}
{Kimi Team}, Angang Du, Bofei Gao, et~al.
\newblock {Kimi k1.5}: Scaling reinforcement learning with {LLMs}.
\newblock \emph{arXiv preprint arXiv:2501.12599}, 2025.
\newblock \doi{10.48550/arXiv.2501.12599}.
\newblock URL \url{https://arxiv.org/abs/2501.12599}.

\bibitem[Kraemer and Banerjee(2016)]{kraemer2016ctde}
Landon Kraemer and Bikramjit Banerjee.
\newblock Multi-agent reinforcement learning as a rehearsal for decentralized planning.
\newblock \emph{Neurocomputing}, 190:\penalty0 82--94, 2016.
\newblock \doi{10.1016/j.neucom.2016.01.031}.

\bibitem[Lambert et~al.(2024)Lambert, Morrison, Pyatkin, et~al.]{lambert2024tulu3}
Nathan Lambert, Jacob Morrison, Valentina Pyatkin, et~al.
\newblock Tulu 3: Pushing frontiers in open language model post-training.
\newblock \emph{arXiv preprint arXiv:2411.15124}, 2024.
\newblock \doi{10.48550/arXiv.2411.15124}.
\newblock URL \url{https://arxiv.org/abs/2411.15124}.

\bibitem[Li et~al.(2023)Li, Hammoud, Itani, Khizbullin, and Ghanem]{li2023camel}
Guohao Li, Hasan Abed Al~Kader Hammoud, Hani Itani, Dmitrii Khizbullin, and Bernard Ghanem.
\newblock {CAMEL}: Communicative agents for ``mind'' exploration of large language model society.
\newblock In \emph{Advances in Neural Information Processing Systems}, volume~36, pages 51991--52008, 2023.
\newblock Main Conference Track.

\bibitem[Liao et~al.(2025)Liao, Wen, Wang, and Zhang]{liao2025marft}
Junwei Liao, Muning Wen, Jun Wang, and Weinan Zhang.
\newblock {MARFT}: Multi-agent reinforcement fine-tuning.
\newblock \emph{arXiv preprint arXiv:2504.16129}, 2025.
\newblock \doi{10.48550/arXiv.2504.16129}.
\newblock URL \url{https://arxiv.org/abs/2504.16129}.

\bibitem[Liu et~al.(2025)Liu, Chen, Liang, Lyu, and Amato]{liu2025magrpo}
Shuo Liu, Tianle Chen, Zeyu Liang, Xueguang Lyu, and Christopher Amato.
\newblock {LLM} collaboration with multi-agent reinforcement learning.
\newblock \emph{arXiv preprint arXiv:2508.04652}, 2025.
\newblock \doi{10.48550/arXiv.2508.04652}.
\newblock URL \url{https://arxiv.org/abs/2508.04652}.

\bibitem[Liu et~al.(2026)Liu, Chen, Amiri, and Amato]{liu2026collm}
Shuo Liu, Tianle Chen, Ryan Amiri, and Christopher Amato.
\newblock Learning decentralized {LLM} collaboration with multi-agent actor critic.
\newblock \emph{arXiv preprint arXiv:2601.21972}, 2026.
\newblock \doi{10.48550/arXiv.2601.21972}.
\newblock URL \url{https://arxiv.org/abs/2601.21972}.

\bibitem[Longpre et~al.(2021)Longpre, Lu, and Daiber]{longpre2021mkqa}
Shayne Longpre, Yi~Lu, and Joachim Daiber.
\newblock {MKQA}: A linguistically diverse benchmark for multilingual open domain question answering.
\newblock \emph{Transactions of the Association for Computational Linguistics}, 9:\penalty0 1389--1406, 2021.
\newblock \doi{10.1162/tacl_a_00433}.
\newblock URL \url{https://aclanthology.org/2021.tacl-1.82/}.

\bibitem[Lu et~al.(2025)Lu, Holleis, Zhang, Aumayer, Nan, Bai, Ma, Ma, Li, Yin, Wang, and Pang]{lu2025toolsandbox}
Jiarui Lu, Thomas Holleis, Yizhe Zhang, Bernhard Aumayer, Feng Nan, Haoping Bai, Shuang Ma, Shen Ma, Mengyu Li, Guoli Yin, Zirui Wang, and Ruoming Pang.
\newblock {ToolSandbox}: A stateful, conversational, interactive evaluation benchmark for {LLM} tool use capabilities.
\newblock In \emph{Findings of the Association for Computational Linguistics: NAACL 2025}, pages 1160--1183. Association for Computational Linguistics, 2025.
\newblock \doi{10.18653/v1/2025.findings-naacl.65}.
\newblock URL \url{https://aclanthology.org/2025.findings-naacl.65/}.

\bibitem[Oliehoek and Amato(2016)]{oliehoek2016decpomdp}
Frans~A. Oliehoek and Christopher Amato.
\newblock \emph{A Concise Introduction to Decentralized {POMDPs}}.
\newblock SpringerBriefs in Intelligent Systems. Springer, 2016.
\newblock \doi{10.1007/978-3-319-28929-8}.

\bibitem[Papoudakis et~al.(2021)Papoudakis, Christianos, and Albrecht]{papoudakis2021liam}
Georgios Papoudakis, Filippos Christianos, and Stefano~V. Albrecht.
\newblock Agent modelling under partial observability for deep reinforcement learning.
\newblock In \emph{Advances in Neural Information Processing Systems}, volume~34, pages 19210--19222, 2021.

\bibitem[Park et~al.(2025)Park, Han, Guo, Ozdaglar, Zhang, and Kim]{park2025maporl}
Chanwoo Park, Seungju Han, Xingzhi Guo, Asuman Ozdaglar, Kaiqing Zhang, and Joo-Kyung Kim.
\newblock {MAPoRL}: Multi-agent post-co-training for collaborative large language models with reinforcement learning.
\newblock \emph{arXiv preprint arXiv:2502.18439}, 2025.
\newblock \doi{10.48550/arXiv.2502.18439}.
\newblock URL \url{https://arxiv.org/abs/2502.18439}.

\bibitem[Patil et~al.(2025)Patil, Mao, Yan, Ji, Suresh, Stoica, and Gonzalez]{patil2025bfcl}
Shishir~G. Patil, Huanzhi Mao, Fanjia Yan, Charlie Cheng-Jie Ji, Vishnu Suresh, Ion Stoica, and Joseph~E. Gonzalez.
\newblock The berkeley function calling leaderboard ({BFCL}): From tool use to agentic evaluation of large language models.
\newblock In \emph{Proceedings of the 42nd International Conference on Machine Learning}, volume 267 of \emph{Proceedings of Machine Learning Research}, pages 48371--48392. PMLR, 2025.
\newblock URL \url{https://proceedings.mlr.press/v267/patil25a.html}.

\bibitem[Qian et~al.(2024)Qian, Liu, Liu, Chen, Dang, Li, Yang, Chen, Su, Cong, Xu, Li, Liu, and Sun]{qian2024chatdev}
Chen Qian, Wei Liu, Hongzhang Liu, Nuo Chen, Yufan Dang, Jiahao Li, Cheng Yang, Weize Chen, Yusheng Su, Xin Cong, Juyuan Xu, Dahai Li, Zhiyuan Liu, and Maosong Sun.
\newblock {ChatDev}: Communicative agents for software development.
\newblock In \emph{Proceedings of the 62nd Annual Meeting of the Association for Computational Linguistics (Volume 1: Long Papers)}, pages 15174--15186, Bangkok, Thailand, 2024. Association for Computational Linguistics.
\newblock \doi{10.18653/v1/2024.acl-long.810}.
\newblock URL \url{https://aclanthology.org/2024.acl-long.810/}.

\bibitem[{Qwen Team}(2025)]{qwen3modelcard}
{Qwen Team}.
\newblock {Qwen3-4B-Instruct-2507}.
\newblock Hugging Face model card, 2025.
\newblock URL \url{https://huggingface.co/Qwen/Qwen3-4B-Instruct-2507}.

\bibitem[Rabinowitz et~al.(2018)Rabinowitz, Perbet, Song, Zhang, Eslami, and Botvinick]{rabinowitz2018tomnet}
Neil Rabinowitz, Frank Perbet, H.~Francis Song, Chiyuan Zhang, S.~M.~Ali Eslami, and Matthew Botvinick.
\newblock Machine theory of mind.
\newblock In \emph{Proceedings of the 35th International Conference on Machine Learning}, volume~80 of \emph{Proceedings of Machine Learning Research}, pages 4218--4227. PMLR, 2018.

\bibitem[Raileanu et~al.(2018)Raileanu, Denton, Szlam, and Fergus]{raileanu2018som}
Roberta Raileanu, Emily Denton, Arthur Szlam, and Rob Fergus.
\newblock Modeling others using oneself in multi-agent reinforcement learning.
\newblock In \emph{Proceedings of the 35th International Conference on Machine Learning}, volume~80 of \emph{Proceedings of Machine Learning Research}, pages 4257--4266. PMLR, 2018.

\bibitem[Sukhbaatar et~al.(2016)Sukhbaatar, Szlam, and Fergus]{sukhbaatar2016commnet}
Sainbayar Sukhbaatar, Arthur Szlam, and Rob Fergus.
\newblock Learning multiagent communication with backpropagation.
\newblock In \emph{Advances in Neural Information Processing Systems}, volume~29, 2016.

\bibitem[Williams(1992)]{williams1992reinforce}
Ronald~J. Williams.
\newblock Simple statistical gradient-following algorithms for connectionist reinforcement learning.
\newblock \emph{Machine Learning}, 8:\penalty0 229--256, 1992.
\newblock \doi{10.1007/BF00992696}.

\bibitem[Wu et~al.(2023)Wu, Bansal, Zhang, Wu, Li, Zhu, Jiang, Zhang, Zhang, Liu, Awadallah, White, Burger, and Wang]{wu2023autogen}
Qingyun Wu, Gagan Bansal, Jieyu Zhang, Yiran Wu, Beibin Li, Erkang Zhu, Li~Jiang, Xiaoyun Zhang, Shaokun Zhang, Jiale Liu, Ahmed~Hassan Awadallah, Ryen~W. White, Doug Burger, and Chi Wang.
\newblock {AutoGen}: Enabling next-gen {LLM} applications via multi-agent conversation.
\newblock \emph{arXiv preprint arXiv:2308.08155}, 2023.
\newblock \doi{10.48550/arXiv.2308.08155}.
\newblock URL \url{https://arxiv.org/abs/2308.08155}.

\bibitem[Yang et~al.(2025)Yang, Li, Yang, Zhang, Hui, Zheng, Yu, Gao, Huang, Lv, et~al.]{qwen3technicalreport}
An~Yang, Anfeng Li, Baosong Yang, Beichen Zhang, Binyuan Hui, Bo~Zheng, Bowen Yu, Chang Gao, Chengen Huang, Chenxu Lv, et~al.
\newblock {Qwen3} technical report.
\newblock \emph{arXiv preprint arXiv:2505.09388}, 2025.
\newblock \doi{10.48550/arXiv.2505.09388}.
\newblock URL \url{https://arxiv.org/abs/2505.09388}.

\bibitem[Yao et~al.(2022)Yao, Zhao, Yu, Du, Shafran, Narasimhan, and Cao]{yao2022react}
Shunyu Yao, Jeffrey Zhao, Dian Yu, Nan Du, Izhak Shafran, Karthik Narasimhan, and Yuan Cao.
\newblock React: Synergizing reasoning and acting in language models.
\newblock \emph{arXiv preprint arXiv:2210.03629}, 2022.

\end{thebibliography}

\newpage
\appendix

\section{Experimental Settings}

\subsection{Tasks and Datasets}

We evaluate on six decentralized, two-agent collaboration tasks. Minecraft tasks are based on the game Minecraft, where agents interact through structured actions on the map of the game; Web tasks use datasets. Agents observe different information sources; Tool tasks use deterministic stateful environments (a simulated phone assistant and an in-memory filesystem). Agents operate through structured tool calls.

\begin{table}[htbp]
\centering
\caption{Dataset splits and primary metrics.}
\label{tab:appendix-datasets}
\footnotesize
\setlength{\tabcolsep}{3pt}
\begin{tabular}{lccc}
\toprule
\textbf{Task} & \textbf{Split} & \textbf{Horizon} & \textbf{Metric} \\
\midrule
Resource Gathering   & $7/7$    & $4$ & success \\
Bridge Building      & $6/3$    & $4$ & success \\
Cross-Lingual        & $120/30$ & $2$ & match score \\
Multi-Source         & $30/15$  & $3$ & match score \\
Stateful Tool        & $25/15$  & $4$ & success \\
Filesystem Wrangling & $22/13$  & $4$ & success \\
\bottomrule
\end{tabular}
\end{table}

\paragraph{Resource Gathering.}
The map of the dataset sizes include $8\times8$, $9\times9$, and $10\times10$. Each cell may contain a tuple of (wood, stone, iron) resources. The task goal is gathering enough required instances. Each turn, an agent's decision is limited to a movement path and at most $2$ extraction targets. The view radius is $2$ and the maximum path length is $4$. Extraction succeeds only when the target cell lies within Manhattan distance $2$ of the agent's final position. The two agents have asymmetric tools: one can harvest wood while the other can harvest stone and iron. Thus, communication is necessary, as an agent may observe resources that are useful to its teammate. Success is defined as achieving the task goal.

\paragraph{Bridge Building.}
All maps of dataset are $13\times11$ grids. The agent view radius is $3$. The task goal is to connects the two anchor regions. The maps differ in things like pillar layout and positions of anchor. Each map contains candidate pillars, including true supports and fake hazardous supports. Agents should discover useful supports through local observation and probing, or communicate probe results and coordinate construction. Each response may probe up to $3$ candidate pillars and issue bridge-construction commands within a command budget. Success is defined as achieving the task goal.

\paragraph{Cross-Lingual Gathering.}
This task is built from a Chinese MKQA subset~\citep{longpre2021mkqa}. Each instance contains paired English and Chinese retrieval corpora and a two-piece support chain. One agent searches English Wikipedia passages, then identify an English page title or bridge clue. The other searches Chinese passages. It is responsible for producing the final answer from target-language evidence. The retriever returns the top $4$ passages, each prompt shows up to $6$ visible pieces, and communication can include up to $3$ evidence items. Success is defined as finding the final answer.

\paragraph{Multi-Source Gathering.}
This task uses HybridQA examples~\citep{chen2020hybridqa}, with training instances drawn from the train split and evaluation instances from the development split. The task is balanced across passage-to-table, table-to-passage, and passage-to-table-to-passage reasoning chains. One agent can inspect passage evidence while the other can inspect table-row evidence. Each instance contains a table, a set of passage pieces, and cross-source support annotations. The final answer often depends on combining one table row with one or more passage pieces. The retriever returns the top $4$ pieces, communication is limited to $3$ evidence item. Success is defined as finding the final answer.

\paragraph{Stateful Tool.}
This task is adapted from Apple ToolSandbox~\citep{lu2025toolsandbox} scenarios. The environment simulates a phone holding hidden state (contacts, messages, reminders, device settings, timestamps). The goal is a natural-language user goal (e.g., checking the latest message, turning off cellular data, or rescheduling a reminder to a specific time). The State Inspector can only retrieve information; the Tool Executor can only modify settings/contacts/reminders, send messages, and submit the final answer. Execution tools require the relevant fact to have been known by the Executor before the current turn. Since only the Inspector can retrieve such facts, the team should hand them off across turns in sequence. Task types include setting toggles, message/reminder lookups, sending messages, contact phone/relationship updates, contact removal, and reminder add/modify/remove. Success is defined as the target action being achieved in the world state.

\paragraph{Filesystem Wrangling.}
This task is a faithful port of the multi-turn \texttt{GorillaFileSystem} tasks in BFCL V3~\citep{patil2025bfcl}. The filesystem is a deterministic in-memory state machine. It can be viewed as a mini MCP: a shared in-memory file server accessed by two scope-restricted clients. The Explorer has read-only tools (\texttt{ls}, \texttt{find}, \texttt{cat}, \texttt{grep}, etc.), while the Editor has write tools (\texttt{mkdir}, \texttt{touch}, \texttt{echo}, \texttt{mv}, \texttt{cp}, \texttt{rm}, etc.). Since the editor cannot list or search the directory tree, target paths can only be obtained from the Explorer's communication. Irreversible operations (\texttt{mv}, \texttt{cp}, \texttt{rm}, \texttt{rmdir}) succeed only if the source path was learned by the Editor in a previous turn. Success is defined by a final-state file tree check, or answer matching for the report family.

\subsection{Structured Context and Auxiliary Targets}

For BCMAAC, each environment exposes compact structured side information $x_{i,t}$ constructed only from information visible to agent $i$ at execution time. It never exposes the global answer nor hidden training labels.

The belief targets vary across tasks. In Resource Gathering, they mark the resource cells that match the agent's tool ability. In Bridge Building, they mark the true supports among pillars. In the Web tasks, they mark the evidence pieces. In the Tool tasks, they mark the required progress slots, such as whether the target action is done.

The partner-action signal is teammate action-type prediction, implemented with a four-class taxonomy in every task. Bridge Building uses place/scout/communicate/no-op. Resource Gathering uses gather/move/communicate/no-op.

Auxiliary labels are used only during training. Notice that at execution time agents keep only their actor and recurrent belief encoder.

\subsection{Training Details}

All base model of agents are initialized from \texttt{Qwen/Qwen3-4B-Instruct-2507}~\citep{qwen3technicalreport, qwen3modelcard}. We train in bfloat16 and sample actions with temperature $0.5$ and top-$p=0.6$. Within each task, MAAC and BCMAAC use the same environments, reward definitions, rollout budgets, and evaluation protocols; the actor and critic learning rates are matched across methods, and BCMAAC additionally trains the belief encoder with a separate context learning rate. The belief encoder uses a hidden dimension of $128$ and the structured value head a hidden dimension of $256$. Table~\ref{tab:appendix-hparams} lists the main training hyperparameters.

\begin{table}[t]
\centering
\caption{Main training hyperparameters. The context learning rate applies to BCMAAC variants.}
\label{tab:appendix-hparams}
\footnotesize
\setlength{\tabcolsep}{3pt}
\begin{tabular}{p{0.22\textwidth}p{0.70\textwidth}}
\toprule
\textbf{Task} & \textbf{Training hyperparameters} \\
\midrule
Resource Gathering   & $180$ epochs; rollout/batch $8/4$; $\gamma=0.90$; actor/critic LR $2.5{\times}10^{-6}$; context LR $3.0{\times}10^{-5}$; max tokens $160$. \\
Bridge Building      & $180$ epochs; rollout/batch $8/4$; $\gamma=0.90$; actor/critic LR $2.5{\times}10^{-6}$; context LR $3.0{\times}10^{-5}$; max tokens $160$. \\
Cross-Lingual        & $30$ epochs; rollout/batch $16/4$; $\gamma=0.95$; actor/critic LR $2.0{\times}10^{-6}$; context LR $1.0{\times}10^{-5}$; max tokens $192$. \\
Multi-Source         & $30$ epochs; rollout/batch $12/2$ or $16/4$; $\gamma=0.95$; actor/critic LR $1.0{\times}10^{-6}$; context LR $1.0{\times}10^{-4}$; max tokens $192$. \\
Stateful Tool        & $120$ epochs; rollout/batch $10/2$; actor/critic LR $1.5{\times}10^{-6}$; context LR $5.0{\times}10^{-5}$; max tokens $160$. \\
Filesystem Wrangling & $120$ epochs; rollout/batch $10/2$; actor/critic LR $1.5{\times}10^{-6}$; context LR $3.0{\times}10^{-5}$; max tokens $160$. \\
\bottomrule
\end{tabular}
\end{table}

\subsection{Reward Designs}

All rewards are computed by deterministic environment rules after parsing the agents' JSON actions. We use dense shaping only to expose intermediate progress; the main evaluation metrics remain the task-native success or answer-match scores.

\paragraph{Resource Gathering.}
Let $p_t$ be normalized collection progress, $c_t$ indicate whether the resource goal is completed, $u_t$ count useful resource communications, $z_t$ and $e_t$ denote move-to-zone and first-zone-entry indicators, and $w_t$ be the number of wasted extraction attempts. The shaped reward is
\begin{equation*}
r^{\mathrm{RG}}_t = 8.0\,\Delta p_t + 7.0\,c_t + 0.1\,u_t + 0.05\,z_t + 0.15\,e_t - 0.1\,w_t .
\end{equation*}

\paragraph{Bridge Building.}
Let $g_t$ be the normalized remaining connection gap, $m_t$ the normalized number of newly merged connected components, $y_t$ the normalized newly connected true-support signal, $n_t$ the normalized number of newly unsafe adjacencies to fake supports, and $b_t$ the normalized number of newly placed blocks; $m_t$, $y_t$, and $n_t$ are normalized by the corresponding per-map totals. The bridge reward is
\begin{equation*}
\begin{aligned}
r^{\mathrm{BB}}_t =&\;5.0(g_{t-1}-g_t)+1.5\,m_t+5.0\,y_t+2.0\,c_t \\
&\;+2.5\,\Delta d_t-12.0\,n_t-5.0\,b_t+q^{\mathrm{late}}_t ,
\end{aligned}
\end{equation*}
where $c_t$ indicates terminal connectivity, $\Delta d_t$ is movement progress toward the opposite side, and $q^{\mathrm{late}}_t$ is a small late-turn quality term favoring cleaner, more compact, and closer-to-finish bridges.

\paragraph{Web Gathering.}
Both Web tasks share the same reward template. Let $a_t$ be the answer-match score, $s_t$ support coverage, $u_t$ useful communication, $x_t$ cross-source communication, $h_t$ informed search, $k_t$ support-chain completion, $b_t$ shared-board completion, $e_t$ answer agreement, and $f_t$ a fully supported answer indicator that fires when a correct answer is submitted with a complete support board. The reward is
\begin{equation*}
\begin{aligned}
r^{\mathrm{Web}}_t =&\;0.2\,\Delta s_t+2.0\,a_t+0.2\,u_t+0.03\,x_t+0.3\,h_t \\
&\;+0.6\,k_t+0.25\,b_t+0.4\,e_t+0.6\,f_t .
\end{aligned}
\end{equation*}
Substring answer matches receive partial credit. Explicit penalties for wrong, premature, conflicting, or unsupported answers are set to zero in the reported runs, relying instead on support and agreement rewards to shape answer readiness.

\paragraph{Stateful Tool.}
Let $\Delta p_t$ be the progress increment (fraction of required progress slots satisfied), $c_t$ indicate terminal success, $u_t,g_t$ useful/grounded communication, $v_t$ the number of valid tool calls, $\mathrm{inv}_t,\mathrm{err}_t$ the numbers of invalid actions / tool errors, $\mathrm{pre}_t$ premature answering, $\rho_t$ the ``ready-but-not-executed/answered'' penalty, $\mathrm{fail}_t$ terminal failure, and $\mathrm{noop}_t$ the number of no-ops. The reward is
\begin{equation*}
\begin{aligned}
r^{\mathrm{ST}}_t =&\;0.8\,\Delta p_t + 4.0\,c_t + 0.08\,u_t + 0.01\,g_t + 0.03\,v_t \\
&\;- 0.08\,\mathrm{inv}_t - 0.12\,\mathrm{err}_t - 0.15\,\mathrm{pre}_t \\
&\;- 0.75\,\rho_t - 0.25\,\mathrm{fail}_t - 0.02\,\mathrm{noop}_t .
\end{aligned}
\end{equation*}

\paragraph{Filesystem Wrangling.}
Using the same notation, with $\mathrm{derr}_t$ the number of irreversible (destructive) operation errors, the reward is
\begin{equation*}
\begin{aligned}
r^{\mathrm{FW}}_t =&\;1.3\,\Delta p_t + 3.3\,c_t + 0.18\,u_t + 0.05\,g_t + 0.06\,v_t \\
&\;- 0.05\,\mathrm{inv}_t - 0.06\,\mathrm{err}_t - 0.12\,\mathrm{derr}_t - 0.10\,\mathrm{pre}_t \\
&\;- 0.45\,\rho_t - 0.15\,\mathrm{fail}_t - 0.01\,\mathrm{noop}_t .
\end{aligned}
\end{equation*}
The largest penalty $\rho_t$ ($0.45$) fires only on terminal failure when all non-terminal prerequisite slots are satisfied but the irreversible action was never executed.

\subsection{Prompt Details}

All environments require agents to return exactly one JSON object, which makes the parser deterministic. And it helps convert natural-language decisions into structured actions. We show abbreviated prompt skeletons below.

\paragraph{Minecraft prompts.}
\begin{lstlisting}
Task setup:
- Current turn: {turn_idx}/{max_turns}
- Goal / progress: {goal_or_connection_state}
- Local state: {position, visible_objects}
- Received teammate messages: {received_messages}

Return exactly one JSON object:
{
  "comm": { ... structured facts ... },
  "probe": [[x,z], ...],
  "cmds": [... extraction targets or bridge commands ...],
  "path": [[x,z], ...]
}
\end{lstlisting}

\paragraph{Web prompts.}
\begin{lstlisting}
Task setup:
- Question: {question}
- Your source: {English / Chinese / passage / table}
- Search history: {search_history}
- Visible evidence pieces: {visible_pieces}
- Shared evidence board: {shared_board}

Return exactly one JSON object:
{
  "comm": {"evidence_ids": ["piece_id"], "notes": "short grounded clue"},
  "search": "query string",
  "answer": ""
}
\end{lstlisting}

\paragraph{Tool prompts.}
\begin{lstlisting}
Task setup:
- Role: {State Inspector / Tool Executor  |  Explorer / Editor}
- Current turn: {turn_idx}/{max_turns}
- User goal: {user_goal}
- Allowed tools: {allowed_tools}
- Known facts (self): {known_facts}
- Recent tool results: {tool_results}
- Received teammate messages: {received_messages}

Return exactly one JSON object:
{
  "comm": {"facts": { ... }, "request": ["needed fact"],
           "intent": "short next-step intent", "ready": false},
  "tool": {"name": "tool_name", "args": { ... }},
  "answer": ""
}
\end{lstlisting}

\subsection{Evaluation Protocol and Metrics}
\label{app:eval-protocol}

In the main results we report task-native metrics: success for Minecraft and Tool, and match score for Web. We report the mean and standard deviation across $5$ independent training runs, where each run's score is averaged over several late-training evaluation checkpoints.

\label{app:eval-metrics}

\textbf{Resource Gathering Success.}
Let $\mathcal{R}$ be the set of resource types and $C_r,G_r$ the final collected and required counts. An episode succeeds iff all resource targets are met:
\[
\mathbf{1}_{\mathrm{RG}}=\mathbf{1}\!\left[\forall r\in\mathcal{R},\; C_r\ge G_r\right],
\qquad
\text{score}=100\cdot \tfrac{1}{N}\sum_{j=1}^{N}\mathbf{1}_{\mathrm{RG}}^{(j)} .
\]

\textbf{Bridge Building Success.}
Let $\mathcal{A}_S,\mathcal{A}_T$ be the two anchor regions and $\mathcal{F}$ the final traversable set (anchors, pillars, placed bridge blocks). If there is a finite path $(p_0,\ldots,p_L)$ with
\[
\begin{gathered}
p_0\in\mathcal{A}_S,\quad p_L\in\mathcal{A}_T,\\
p_\ell\in\mathcal{F}\;(\ell=0,\ldots,L),\quad \|p_{\ell+1}-p_\ell\|_1=1\;(\ell<L),
\end{gathered}
\]
then $\mathrm{Conn}_{\mathcal{F}}(\mathcal{A}_S,\mathcal{A}_T)$ holds (four-neighbor connectivity). The success indicator is $\mathbf{1}_{\mathrm{BB}}=\mathbf{1}[\mathrm{Conn}_{\mathcal{F}}(\mathcal{A}_S,\mathcal{A}_T)]$ and the score is $100\cdot\tfrac1N\sum_j\mathbf{1}_{\mathrm{BB}}^{(j)}$.

\textbf{Cross-Lingual Gathering Match Score.}
Let $n(\cdot)$ denote answer normalization (lowercasing and removal of superficial formatting differences) and $\mathcal{G}$ the gold answer and its annotated aliases. For a submitted answer $\hat y$ and a gold answer or alias $y$,
\[
s_{\mathrm{CL}}(\hat y,y)=
\begin{cases}
1, & n(\hat y)=n(y),\\
\alpha, & \text{one normalized answer contains the other},\\
0, & \text{otherwise},
\end{cases}
\]
where $\alpha=0.7$ is the substring-match credit. Let $\widehat{\mathcal{Y}}_t$ be the set of non-empty answers submitted at turn $t$; if none is submitted then $m^{\mathrm{CL}}_t=0$, otherwise
\[
m^{\mathrm{CL}}_t=\kappa_t\max_{\hat y\in\widehat{\mathcal{Y}}_t} q_{\mathrm{CL}}(\hat y),
\quad
q_{\mathrm{CL}}(\hat y)=\max_{y\in\mathcal{G}} s_{\mathrm{CL}}(\hat y,y),
\]
with $\kappa_t=1$ unless multiple submitted answers conflict, in which case $\kappa_t$ is a configured down-weighting factor. The episode score is $M_{\mathrm{CL}}=\sum_{t=0}^{H-1}m^{\mathrm{CL}}_t$; since submitting an answer terminates the episode, at most one term is nonzero during evaluation. The reported score is $100\cdot\tfrac1N\sum_j M_{\mathrm{CL}}^{(j)}$.

\textbf{Multi-Source Gathering Match Score.}
This uses the same per-turn accumulation protocol but matches against the task's final annotated answer $y^\star$:
\[
s_{\mathrm{MS}}(\hat y,y^\star)=
\begin{cases}
1, & n(\hat y)=n(y^\star),\\
\alpha, & n(\hat y)\ \text{contains}\ n(y^\star),\\
0, & \text{otherwise},
\end{cases}
\]
with $\alpha=0.7$ and the same down-weighting rule, $M_{\mathrm{MS}}=\sum_t m^{\mathrm{MS}}_t$, and reported score $100\cdot\tfrac1N\sum_j M_{\mathrm{MS}}^{(j)}$.

\textbf{Stateful Tool / Filesystem Wrangling Success.}
Tool tasks use a state-verified binary success. Let $\mathrm{done}_t$ indicate whether the target action is achieved in the world state after turn $t$. An episode succeeds if any turn achieves the target:
\[
\mathbf{1}_{\mathrm{Tool}}=\mathbf{1}\!\left[\exists\,t\le H-1:\ \mathrm{done}_t\right],
\qquad
\text{score}=100\cdot\tfrac1N\sum_{j=1}^{N}\mathbf{1}_{\mathrm{Tool}}^{(j)} .
\]

\subsection{Baseline Details}
\label{app:baseline-details}

All methods are evaluated under the same task instances, environment transitions, reward definitions, parsing rules, and task-native metrics. The differences are the training signal, whether multiple agents are used, and whether the actor/critic receives belief context. Table~\ref{tab:appendix-baseline-details} summarizes the main comparisons. The prompt-only baselines use the same action format as the trained agents; MAAC and BCMAAC use the same rollout budgets, base LLM initialization, actor/critic optimization settings, and environment rewards; the ablations change only the auxiliary supervision attached to the belief encoder.

\begin{table}[t]
\centering
\caption{Detailed definitions of baselines and ablation variants.}
\label{tab:appendix-baseline-details}
\footnotesize
\setlength{\tabcolsep}{4pt}
\begin{tabular}{p{0.13\textwidth}p{0.22\textwidth}p{0.22\textwidth}p{0.33\textwidth}}
\toprule
\textbf{Method} & \textbf{Execution} & \textbf{Training signal} & \textbf{Notes} \\
\midrule
Raw Model & Decentralized LLM agents without RL updates. & None. & Uses the same prompts and parsers as the trained methods; measures the capability of the initial instruction model. \\
AC & Single-agent actor-critic. & Return regression and policy-gradient updates from the task reward. & Tests how far task-specific RL improves a single actor without learned multi-agent coordination. \\
Parallel & Prompt-only multi-agent parallel execution. & None. & Agents act in parallel from their own prompts and received messages; no policy update or centralized critic. \\
Pipeline & Prompt-only serial execution. & None. & Agents act in a fixed order, so later agents can condition on earlier outputs; the collaboration pattern is manually specified. \\
MAAC & Decentralized actors with centralized training. & Actor-critic updates with a centralized critic over joint histories. & The closest RL baseline; actors execute from local histories while the critic estimates team return during training. \\
BCMAAC-P & BCMAAC without task-belief supervision. & Actor-critic updates plus partner-action / action-preference auxiliary supervision. & Isolates the effect of modeling teammate behavior / action preferences. \\
BCMAAC-B & BCMAAC without partner-action supervision. & Actor-critic updates plus task-belief supervision. & Isolates the effect of grounding the context in task-state variables; in Web/Tool tasks answer-readiness is grouped with task belief. \\
BCMAAC & Belief-conditioned decentralized actors with a structured centralized critic. & Actor-critic updates plus task-belief and partner-action auxiliary signals. & The full method used for the main comparison. \\
\bottomrule
\end{tabular}
\end{table}

\begin{table}[t]
\centering
\caption{Structured side information and task-belief targets used by BCMAAC.}
\label{tab:appendix-side-info}
\footnotesize
\setlength{\tabcolsep}{4pt}
\begin{tabular}{p{0.18\textwidth}p{0.27\textwidth}p{0.16\textwidth}p{0.31\textwidth}}
\toprule
\textbf{Task} & \textbf{Structured grid} & \textbf{Scalars} & \textbf{Task-belief target} \\
\midrule
Resource Gathering   & $10$ spatial channels over the map. & $11$ normalized scalar features. & Non-empty resource cells relevant to the agent's tool role. \\
Bridge Building      & $12$ spatial channels over the map. & $8$ normalized scalar features. & True support pillars among candidate pillars. \\
Cross-Lingual Gathering & $8$ piece channels over the evidence list. & $20$ normalized scalar features. & Evidence pieces in the annotated support chain. \\
Multi-Source Gathering  & $8$ piece channels over table and passage pieces. & $20$ normalized scalar features. & Evidence pieces in the annotated cross-source support chain. \\
Stateful Tool        & $8$ channels over $6$ belief slots. & $16$ normalized scalars. & Required progress slots (whether the executor holds the key fact, whether the setting is ready, whether the target action is done, etc.). \\
Filesystem Wrangling & $8$ channels over $6$ belief slots. & $16$ normalized scalars. & Required progress slots (whether the target is located, whether the executor knows the path, whether the destination is ready, whether the target action is done, etc.). \\
\bottomrule
\end{tabular}
\end{table}

\subsection{Structured Side Information}
\label{app:structured-side-info}

Recall that the side information $x_{i,t}$ is built only from information available to agent $i$ at execution time. The following examples are schematic; the implementation converts these fields into grid channels and scalar vectors.

\begin{lstlisting}
Resource Gathering:
{ "grid": {visibility, visible_wood/stone/iron, msg_wood/stone/iron,
           self_pos, teammate_pos, extraction_reach},
  "scalars": {turn_ratio, normalized_position, goal[w,s,i],
              collected_ratio[w,s,i], inbox_load, completed} }
\end{lstlisting}
\begin{lstlisting}
Bridge Building:
{ "grid": {visibility, visible_land, visible_filled, visible_candidate,
           visible_anchor_s/t, known_y/n, msg_y/n/candidate, self_pos},
  "scalars": {turn_ratio, normalized_position, known_safe_count,
              known_unsafe_count, message_count, connected, terminated} }
\end{lstlisting}
\begin{lstlisting}
Cross-Lingual / Multi-Source (piece grid):
{ "piece_grid": {seen_by_self, on_shared_board, in_last_results, in_inbox,
                 question_overlap, hit_count, public_hint_overlap, own_source},
  "scalars": {turn_ratio, remaining_turn_ratio, seen_budget,
              shared_board_budget, inbox_backlog / unresolved_backlog,
              search_depth / recent_board_growth, query_repeat_ratio /
              local_unshared_visible_ratio, board_has_both_sources, ...} }
\end{lstlisting}
\begin{lstlisting}
Stateful Tool / Filesystem Wrangling (slot grid):
{ "grid": {self_knows_slot, inbox_mentions_slot, last_tool_relates_to_slot,
           slot_in_read_set, slot_in_write_set, slot_required,
           success_flag, turn_fraction},   # 8 channels over belief slots
  "scalars": {turn_ratio, role_onehot, per_fact_known_flags,
              inbox_size, tool_history_len, success, ...} }  # 16-dim
\end{lstlisting}

Table~\ref{tab:appendix-side-info} summarizes the tensor dimensions and task-belief targets.

\subsection{Auxiliary Label Generation}
\label{app:aux-label-generation}

The labeling rule $\Gamma_{\xi}(\tau)$ produces all auxiliary labels after each rollout, reading the completed trajectory together with the environment's internal state. The task-belief loss is a masked binary cross-entropy,
\[
\mathcal L_b=\mathrm{BCEWithLogits}\!\left(\hat{\mathbf b}_{i,t}[\boldsymbol\omega_{i,t}],\ \mathbf b^*_{i,t}[\boldsymbol\omega_{i,t}]\right),
\]
where $\mathbf b^*_{i,t}$ marks the ground-truth entries listed in Table~\ref{tab:appendix-side-info}, and the mask $\boldsymbol\omega_{i,t}$ restricts supervision to unresolved or required entries. For the Web and Tool tasks, an answer-readiness label is grouped with task belief: the target is $1$ at the first successful turn and $0$ otherwise. The partner-action loss is a cross-entropy over a task-specific four-class action taxonomy,
\[
\mathcal{L}_p = \mathrm{CE}(\hat{\mathbf{p}}_{i,t},\,a^*_{-i,t}),
\]
where $a^*_{-i,t}$ is the parsed teammate action type; Minecraft tasks add action-preference supervision over the \texttt{comm}, \texttt{probe}, \texttt{cmds}, and \texttt{path} fields. The full auxiliary objective is
\[
\mathcal L_{\mathrm{aux}}=\lambda_b\mathcal L_b+\lambda_p\mathcal L_p+\lambda_r\mathcal L_r+\lambda_a\mathcal L_a,
\]
where $\mathcal L_r$ is the answer-readiness loss and $\mathcal L_a$ the action-preference loss; in the main text, $\mathcal L_r$ is grouped into the task-belief term and $\mathcal L_a$ into the partner-action term. BCMAAC-B sets $\lambda_p$ and $\lambda_a$ to zero; BCMAAC-P sets $\lambda_b$ and $\lambda_r$ to zero.

\section{Proofs}
\label{app:proofs}

\begin{proof}[Proof of Proposition~\ref{prop:sufficiency}]
Because $\mathbf c_t$ is determined by information in $\mathcal{F}_t$, conditioning on $\mathcal{F}_t$ also determines $\mathbf c_t$. Hence
\begin{align*}
V^{\boldsymbol{\pi}}(\mathcal{F}_t)
&=\mathbb{E}_{\boldsymbol{\pi}}[G_t\mid \mathcal{F}_t]\nonumber\\
&=\mathbb{E}_{\boldsymbol{\pi}}[G_t\mid \mathbf c_t]
=V^{\boldsymbol{\pi}}(\mathbf c_t).
\end{align*}
\end{proof}

\end{document}